\documentclass{article}
\usepackage{arxiv} 

\usepackage[T1]{fontenc}
\usepackage{graphicx}
\usepackage{amsmath}
\usepackage{amssymb}
\usepackage{amsthm}
\usepackage{thmtools}
\usepackage{thm-restate}
\usepackage{tikz}
\usepackage{mathtools}
\usepackage{amsfonts}
\usepackage{hyperref}
\usepackage{cleveref}
\usepackage{xcolor}
\usepackage{subcaption}
\usepackage{url}
\usepackage{times}
\usepackage[vlined,linesnumbered,ruled]{algorithm2e}
\usepackage{color}
\newtheorem{theorem}{Theorem}[section]
\newtheorem{lemma}[theorem]{Lemma}

\hypersetup{
    colorlinks=true,
    linkcolor=blue,
    urlcolor=blue,
    citecolor=blue,
    anchorcolor=blue,
    filecolor=blue,
    menucolor=blue,
    runcolor=blue,
    linkbordercolor=blue,
    urlbordercolor=blue,
    citebordercolor=blue,
    filebordercolor=blue,
    menubordercolor=blue,
    runbordercolor=blue,
    pdfborderstyle={/S/U/W 1}
}

\newcommand{\realnum}{\mathbb{R}}

\newcommand{\natnum}{\mathbb{N}}

\title{Analysis of Evolutionary Diversity Optimisation for the Maximum Matching Problem}
\author{
    Jonathan Gadea Harder \\
    Algorithm Engineering\\
    Hasso Plattner Institute\\ University of Potsdam\\
    Potsdam, Germany
    \And
    Aneta Neumann \\
   Optimisation and Logistics\\
School of Computer and Mathematical Sciences\\
University of Adelaide, Australia\\ 
Adelaide, SA 5005\\ 
Australia 
    \And
    Frank Neumann \\
    Optimisation and Logistics\\
School of Computer and Mathematical Sciences\\
University of Adelaide, Australia\\ 
Adelaide, SA 5005\\ 
Australia 
}

\begin{document}

\maketitle

\begin{abstract}
This paper delves into the enhancement of solution diversity in evolutionary algorithms (EAs) for the maximum matching problem, with a particular focus on complete bipartite graphs and paths. We utilize binary string encoding for matchings and employ Hamming distance as the metric for measuring diversity, aiming to maximize it. Central to our research is the $(\mu+1)$-EA and 2P-EA$_D$, applied for diversity optimization, which we rigorously analyze both theoretically and empirically. 

For complete bipartite graphs, our runtime analysis demonstrates that, for reasonably small $\mu$, the $(\mu+1)$-EA achieves maximal diversity with an expected runtime of \(O(\mu^2 m^4\log(m))\) for the small gap case (where the population size $\mu$ is less than the difference in the sizes of the bipartite partitions) and \(O(\mu^2 m^2\log(m))\) otherwise. For paths we give an upper bound of $O(\mu^3m^3)$. Additionally, for the 2P-EA$_D$ we give stronger performance bounds of \(O(\mu^2 m^2\log(m))\) for the small gap case, \(O(\mu^2 n^2\log(n))\) otherwise, and \(O(\mu^3m^2)\) for paths. Here \(n\) is the total number of vertices and \(m\) the number of edges.
Our empirical studies, examining the scaling behavior with respect to $m$ and $\mu$, complement these theoretical insights and suggest potential for further refinement of the runtime bounds.
\end{abstract}

\section{Introduction}

Evolutionary algorithms (EAs) stand as a robust class of heuristics that navigate the intricate landscapes of various domains, from combinatorial optimization to bioinformatics, and have proven especially valuable in addressing problems within graph theory \cite{NeumannWitt2010}. Central to the discussion in the field is the concept of diversity within EAs, which has been pivotal in enhancing the search process and preventing premature convergence on suboptimal solutions \cite{FriedrichOlivetoSudholtWitt2009}. 
\subsection{Related work}
Recent research in evolutionary computation investigates various connections between quality and diversity. Quality Diversity (QD) has gained recognition as a widely adopted search paradigm, particularly in the fields of robotics and games \cite{DBLP:journals/firai/PughSS16,DBLP:journals/tec/CullyD18,DBLP:conf/cig/GravinaKLTY19,DBLP:conf/cig/AlvarezDFT19,DBLP:journals/tec/BossensT21}. The goal of QD is to illuminate the space of solution behaviours by exploring various niches in the feature space and maximizing quality within each specific niche. In particular, the popular MAP-elites algorithm divides the search space into cells to identify the solution with the highest possible quality for each cell \cite{mouret2015illuminating,vassiliades2017using,DBLP:conf/cig/AlvarezDFT19,DBLP:journals/cim/ZhangCXBZ23}

The area of Evolutionary diversity optimization (EDO) aims to find a maximal diverse set of solutions that all meet a given quality criterion. EDO approaches have been applied in a wide range of settings. 
Diversity, while typically a means to avoid stagnation in the search for a single optimal solution, here is leveraged to yield a set of diverse, high-quality solutions. This is advantageous for decision-makers who value a variety of options from which to select the most fitting solution, accounting for different practical considerations and trade-offs \cite{UlrichBaderThiele2010,UlrichThiele2011}.
For example the use of different diversity measures has been explored for evolving diverse set of TSP instances that exhibit the difference in performance of algorithms for the traveling salesperson problem as well as differences in terms of features of variation of a given image.\cite{TSP}
In the classical context of combinatorial optimization, EDO algorithms have been designed for problems such as the knapsack problem~\cite{knapsack}, the computation of minimum spanning trees~\cite{mst}, communication networks~\cite{EDOSDCN,DBLP:conf/gecco/NeumannGYSCG023}, to compute sets of problem instances~\cite{DBLP:journals/ec/GaoNN21,DBLP:conf/gecco/NeumannGDN018,DBLP:conf/gecco/NeumannG0019}, as well as the computation of diverse sets of solutions for monotone submodular functions under given constraints~\cite{submodular,DBLP:conf/ijcai/DoGN023}. Furthermore, Pareto Diversity Optimization (PDO) has been developed in~\cite{DBLP:conf/gecco/NeumannA022} which is a coevolutionary approach optimizing the quality of the best possible solution as well as computing a diverse set of solutions meeting a given threshold value. EDO approaches have been analyzed with respect to their theoretical behavior for simple single- and multi-objective pseudo-Boolean functions~\cite{multiobjective} as well as simple scenarios of the traveling salesperson problem~\cite{TSP,DBLP:conf/gecco/NikfarjamBN021,DBLP:conf/foga/NikfarjamB0N21}, the minimum spanning tree problem~\cite{mst}, the traveling thief problem~\cite{DBLP:conf/gecco/NikfarjamN022a}, the permutation problems~\cite{DBLP:journals/telo/DoGNN22} and the optimization of submodular functions~\cite{submodular}. 

\subsection{Our contribution}
This paper builds upon the methodology of \cite{frank} applying the theoretical runtime analysis framework to the maximum matching problem, specifically in bipartite graphs and paths. We aim to provide a deeper understanding of how diversity mechanisms influence the efficiency of population-based EAs in converging to a diverse set of high-quality maximum matchings.

To achieve this, we adopt a binary string representation for matchings and use Hamming distance as a measure of diversity. We then delve into the theoretical underpinnings of evolutionary diversity optimization for the maximum matching problem, examining structural properties that impact the performance of diversity-enhancing mechanisms within EAs. We provide runtime analysis for evolutionary algorithms, shedding light on their scalability for different problem instances.
Finally, we present our experimental investigations to assess how close the bounds on the theoretical runtimes match the the experimental runtimes.

In summary, our research provides theoretical insights and empirical evidence to understand how diversity can be effectively maximized for the maximum matching problem. Our findings contribute to a deeper understanding of the interplay between diversity and optimization in EAs and pave the way for further research in this direction.

The paper is organized as follows. In Section~\ref{sec2}, we introduce the maximum matching problem and the evolutionary diversity optimization approaches analyzed in this study. We then explore structural properties and present runtime analyses for diversity optimization in the context of complete bipartite graphs and paths (\Cref{sec3}). Experimental investigations are detailed for both unconstrained and constrained scenarios (\Cref{sec4} and ~\ref{sec5}), followed by concluding remarks and suggestions for future research directions (\Cref{sec6}).

\section{Preliminaries}
\label{sec2}
In this part of the paper, we present the core concepts related to diversity optimization for matchings in bipartite graphs. We start by establishing the definitions and measures of diversity that will be used throughout our discussion.

\subsection{Maximum matching problem and diversity optimization}
Our study is concerned with the matching problem in bipartite graphs, described by a graph \( G = (V, E) \). The aim is to find a maximum matching \( M \), which is a collection of edges that do not share common vertices. It is presumed that each individual in the starting population represents a valid maximum matching. Our analysis is directed at determining how long it takes evolutionary algorithms to cultivate a population that is not only diverse but also meets a specified quality benchmark.

Let \( x \in \{0,1\}^{|E|} \) represent a bitstring where each bit corresponds to an edge in \( E \), indicating whether the edge is included in the matching. We define the fitness function \( f(x) \) as follows, adapting the approach introduced by Giel and Wegener\cite{GielWegener2003}:

\[
f(x) = \begin{cases} 
    -col(x) & \text{if } x \text{ represents an invalid matching} \\
    |x| & \text{if } x \text{ represents a valid matching}
\end{cases}
\]

Here, \( col(x) \) is the collision number, representing the count of pairs of edges that are included in \( x \) and share a common endpoint, rendering \( x \) an invalid matching, and \( |x| \) is the number of edges included in the matching represented by \( x \).

This fitness function imposes a penalty for invalid matchings proportional to the number of edge conflicts, thereby encouraging the evolution of valid matchings. The goal is to maximize \( f(x) \), which aligns with identifying a maximum matching that has no edge collisions.

The divergence between individuals is gauged using the Hamming distance, which is appropriate given our binary string representation of solutions. This distance measures how many bits differ between two strings.

\subsection{Diversity measure}
The diversity of a multiset (duplicates allowed) of search points $P$ (called population in the following) is defined as the cumulative Hamming distance across all unique individual pairings within \( P \). This is mathematically expressed as 
$$ D(P) = \sum_{(x,y) \in \tilde{P} \times \tilde{P}} H(x, y), $$

where \( \tilde{P} \) is the set (no duplicates) containing all solutions in \( P \), and \( H(x, y) \) is the Hamming distance between any two solutions \( x \) and \( y \).
The notion of contribution for a solution \( x \) within a population is quantified as the difference in diversity if \( x \) were to be excluded and defined as
$$c(x) = D(P) - D(P \setminus \{x\}).$$

\subsection{Algorithms}
The (\( \mu+1 \))-EA$_D$ (see Algorithm~\ref{alg:ead}) operates on a principle of maintaining and enhancing diversity within a population. It starts with a population of solutions, iteratively evolving them through mutation. In each iteration, it selects a solution uniformly at random, applies mutation, and if the new solution meets quality criteria, it is added to the population. To maintain population size, the least diverse individual (or one of them, if there are several) is removed. This process continues until the termination criterion is met. In our case this would be achieving maximal diversity and the quality criterion being a valid maximum matching.

\begin{algorithm}[t]
\SetAlgoLined
\DontPrintSemicolon
\caption{($\mu+1$)-EA$_D$}
\label{alg:ead}
\KwIn{A population size $\mu$, individual length $m$, mutation probability $1/m$}
\KwOut{A diverse population of solutions $P$}
Initialize $P$ with $\mu$ $m$-bit binary strings\;
\While{termination criterion not met}{
  Choose $s \in P$ uniformly at random\;
  Produce $s'$ by flipping each bit of $s$ with probability $1/m$ independently\;
  \If{$s'$ meets the quality criteria}{
    Add $s'$ to $P$\;
      Choose a solution $z \in P$ where $c(z) = \min\limits_{x \in P} c(x)$ u.a.r.\;
  Set $P := P \setminus \{z\}$\;
  }

}

\end{algorithm}
The Two-Phase Matching EA$_D$ (see Algorithm~\ref{alg:two_phase_matching}) is also designed to generate diverse solutions in the population. The first phase involves 'unmatching' a random subset of vertices in a solution, while the second phase focuses on 'rematching' these vertices to other unmatched vertices in the graph. The algorithm keeps adding these newly formed solutions to the population if they fulfill the quality criteria and, similar to the (\( \mu+1 \))-EA$_D$, removes the least diverse solutions to maintain population size. The algorithm continues this process until the set criteria are met, aiming to achieve a diverse set of high-quality matchings.

\begin{algorithm}[!t]
\SetAlgoLined
\DontPrintSemicolon
\caption{Two-Phase Matching EA$_D$ (2P-EA$_D$)}
\label{alg:two_phase_matching}
\KwIn{A population size $\mu$, individual length $m$}
\KwOut{A diverse population of solutions $P$}
Initialize $P$ with $\mu$ $m$-bit binary strings;\\
\While{termination criterion not met}{
  Choose $s\in P$ uniformly at random (u.a.r);\\
  Create $s'$ as duplicate of $s$;\\
  Select a subset of vertices $S \subseteq V,$ where each vertex is included with probability $\frac{1}{|V|}$;\\
  \ForEach{vertex $v \in S$}{
    Unmatch $v$ in $s'$ by setting corresponding bits to 0\;
  }
  \ForEach{vertex $v \in S$}{
    \If{there are unmatched neighbors}{
      Match $v$ in $s'$ u.a.r with unmatched neighbor\;
    }
  }
  \If{$s'$ meets quality criteria}{
    Add $s'$ to $P$\;
   Choose a solution $z \in P$ where $c(z) = \min\limits_{x \in P} c(x)$ u.a.r.\;
  }
}
\end{algorithm}

\subsection{Drift theorems}
We analyse the considered algorithms with respect to their runtime behaviour. The expected runtime refers to the expected number of generated offspring until a given goal has been achieved (usually until a valid population of maximal diversity has been computed).
For our analysis, we make use of the additive and multiplicate drift theorems which we state in the following. 
\begin{theorem}[Additive Drift Theorem\cite{driftthm}] \label{thm:additivedrift}
Let $S \subseteq \realnum$ be a finite set of positive numbers and let $({X^t})_{t\in \natnum}$ be a sequence of random variables over $S \cup \{0\}$. Let $T$ be the random variable that denotes the first point in time $t \in \natnum$ for which $X^t \leq 0$.
Suppose that there exists a constant $\delta_1 > 0$ such that

$$ E[X^{t} - X^{t+1} \mid T > t] \geq \delta_1 $$
holds. Then

$$E[T \mid X^0] \leq \frac{X^0}{\delta_1}.$$

If there exists a constant $\delta_2 > 0$ such that

$$ E[X^{t} - X^{t+1} \mid T > t] \leq \delta_2 $$
holds. Then

$$E[T \mid X^0] \geq \frac{X^0}{\delta_2}.$$

\end{theorem}
\begin{theorem}[Multiplicative Drift Theorem\cite{multdrift}]\label{thm:multiplicativedrift}
Let $(X_t)_{t \in \natnum}$ be random variables over $\realnum$, $x_{\min} >0$, and let $T = \min\{t \mid X_t < x_{\min}\}$. Furthermore, suppose that\\

(a) $X_0 \geq x_{\min}$ and, for all $t \leq T$, it holds that $X_t \geq 0$, and that

(b) there is some value $\delta >0$ such that, for all $t < T$, it holds that $X_t - E[X_{t+1} \mid X_0, \ldots, X_t] \geq \delta X_t$.\\

Then

$$ E[T \mid X_0] \leq \frac{1+\ln\left(\frac{X_0}{x_{\min}}\right)}{\delta}.$$
\end{theorem}

\section{Runtime Analysis for complete bipartite graphs}
\label{sec3}
This section introduces key theoretical results on complete bipartite graphs. We commence with a lemma that characterizes the conditions for maximal diversity within a population. Subsequently, we present a series of theorems that delineate the expected runtime to achieve this optimal diversity. These theorems compare the performance of the \((\mu+1)\)-EA\(_D\) and 2P-EA\(_D\) algorithms, providing a quantitative basis for assessing their efficacy.

\begin{lemma}[Diversity of a Population]
Maximal diversity \( D(P) \) on a complete bipartite graph $((G=(L,R),E)$ for a population $P$ of size $\mu<\frac{|R|}{2}$, is attained if and only if all matchings in \( P \) are pairwise edge-disjoint.
\end{lemma}
\begin{proof}
Consider a set of matchings in \( G \), where each matching is a solution in the population. Let the diversity of this set be denoted by \( D \), defined as the sum of pairwise Hamming distances between all matchings.

A matching in \( G \) involves pairing each vertex in \( R \) with a unique vertex in \( L \), yielding \( |R| \) edges in each matching. The Hamming distance between any two distinct matchings is the count of edges that differ between them.

To maximize \( D \), each pair of matchings should differ by the greatest number of edges. This maximum difference is \( |R| \), occurring when the matchings share no common edges.

Given \( \mu \) matchings, the number of distinct pairs of matchings is \( \binom{\mu}{2} \). If all matchings are disjoint, each pair contributes \( 2|R| \) to \( D \), leading to \( D = \mu (\mu-1) |R| \).

If any pair of matchings shares at least one edge, the Hamming distance for that pair is strictly less than \( |R| \), thus reducing \( D \). Therefore, \( D \) is maximized if and only if all \( \mu \) matchings are pairwise edge-disjoint.

This argument hinges on the fact that \( \mu < \frac{|R|}{2} \), ensuring the feasibility of having disjoint matchings in \( G \) since each matching uses \( |R| \) edges and there are \( |R||L| \) possible edges in \( G \). Consequently, it is possible to construct \( \mu \) disjoint matchings, each utilizing a different subset of \( |R| \) edges from the total pool.
\end{proof}
In the following theorem we show that there is always a local improvement, needing $2$ bit flips,  to reach a population with maximum diversity if the difference in size between both partitions is larger than the population size.
\begin{theorem}\label{thm:complete_bip_graph_big_gap}
Let \( G=((L,R),E) \) be a complete bipartite graph with \( \mu < \frac{|R|}{2} \), \( \mu < |L|-|R| \) and \( |R|<|L| \). In the \( (\mu + 1) \)-EA$_D$ applied to \( G \), the expected time until the diversity is maximized is \( O(\mu^2 m^2\log(m)) \).
\end{theorem}
\begin{proof}
We define the potential function \( X_t \) as the difference between the optimal diversity \( \text{div}_{\text{opt}} \) and the current diversity \( \text{div}(t) \) at time \( t \):
\[ X_t\coloneq \text{div}_{\text{opt}} - \text{div}(t). \]

In each solution, exactly \( |R| \) vertices from \( L \) are adjacent to a matching edge, leaving \( |L|-|R| \) vertices in \( L \) unadjacent in every solution. Additionally, each vertex in \( R \) can be matched to at most \( \mu<|L|-|R| \) different vertices across all solutions, ensuring that, for each vertex in R, there exists a vertex in L that is not matched with it in any solution.

To show that there is always a 2-bit flip which improves diversity by at least $\frac{X_t}{\mu} $, we focus on a sequence of improving \( 2 \)-bit flips. Each \( 2 \)-bit flip corresponds to changing a match for a vertex in \( R \), which entails deactivating one edge (currently part of a matching) and activating another edge (currently not part of the matching). This process is akin to reassigning a vertex in \( R \) to a different, unmatched vertex in \( L \). 

Consider an edge \( e \) used in \( i \) solutions. When this edge is deactivated (removed from the matching), the diversity change is \( -(\mu - i) \), since \( \mu - i \) solutions lose a unique edge, reducing diversity. Conversely, when a new edge is activated (added to the matching) that is unused across all other solutions,it contributes \( \mu - 1 \) to the diversity.

Thus, for each such \( 2 \)-bit flip involving edge \( e \), the total change in diversity is:
\[ -(\mu - i) + (\mu - 1) = -\mu + i + \mu - 1 = i - 1. \]
This calculation demonstrates that the diversity improve achieved by applying the \( 2 \)-bit flip for an edge in the sequence either decreases or remains unchanged if it is flipped later in the sequence. Note that in each step of the sequence the new maximum matching contains an edge unused by any other matching, so the offspring is always valid and the diversity improvement is at least $1$, since this would be achieved by replacing the parent. Also since as soon as all edges are unique across all solutions the population is optimal and thus the total change across all such edges equals the difference to the optimum $X_t$.

Let \( \overline{e} \) represent the count of such "imperfect" edges (edges used in more than one solution). Applying the 2-bit flip to one edge of the sequence gives at-least the diversity increase it achieves in the sequence, since the value of $i$ can only decrease or remain unchanged, and it is at most $\mu$. Thus $\overline{e}\mu \geq X_t$, which implies $\overline{e}\geq\frac{X_t}{\mu}$. The expected drift then is:
\begin{align*}
E[X_{t} - X_{t+1} \mid X_t] &\geq \frac{\overline{e}}{\mu m^2}\left(1-\frac{1}{m}\right)^{m-2} \geq \frac{X_t}{\mu^2 m^2e}.
\end{align*}

Given that $\binom{\mu}{2}2|R|$ is the maximum diversity, when all edges aire pairwise distinct, it holds that \( X_0 \leq \binom{\mu}{2}2|R|\leq \mu^2|R| \leq |R|^3\leq m^{1.5} \), the application of the multiplicative drift theorem yields the expected runtime of \( O(\mu^2 m^2\log(m)) \) to achieve maximum diversity.
\end{proof}

We now show that the Two-Phase Matching Algorithm achieves significant speedup since no longer two edges have to be flipped to change where one vertex is matched to.
\begin{theorem}\label{thm:two_phase_matching}
Let \( G=((L,R),E) \) be a complete bipartite graph with \( \mu < \frac{|R|}{2} \), \( \mu < |L|-|R| \) and \( |R|<|L| \). In the Two-Phase Matching Evolutionary Algorithm applied to \( G \), the expected time until the diversity is maximized is \( O(\mu^2 n^2\log(n)) \), where \( n = |L|+|R| \).
\end{theorem}

\begin{proof}
We define the potential function \( X_t \) as the difference between the optimal diversity \( \text{div}_{\text{opt}} \) and the current diversity \( \text{div}(t) \) at time \( t \):
\[ X_t\coloneq \text{div}_{\text{opt}} - \text{div}(t). \]

The maximal diversity is achieved when all matchings in the population are pairwise edge-disjoint. The drift in the potential function \( X_t \) at each step of the algorithm is analyzed as follows:

In each step, the algorithm first selects a solution and a subset of vertices, which it rematches with unmatched vertices in \( L \). Let \( \overline{e} \) represent the count of such "imperfect" edges (edges used in more than one solution). As shown in  \Cref{thm:complete_bip_graph_big_gap} it holds that $\overline{e}\geq\frac{X_t}{\mu}$. The expected drift then is obtained by selecting the corresponding solution to any of the $\overline{e}$ edges, unmatching the adjacent vertex in $R$ and rematching it to include an edge unused by any solution. The probability to unmatch any and no other particular vertex in \( R \) is \( \frac{1}{n}(1-\frac{1}{n})^{n-1}\geq \frac{1}{en}\), and the probability of matching it with an appropriate unmatched vertex in \( L \) is at-least \( \frac{1}{n} \).

The expected decrease in the potential function \( X_t \) per step, or the expected drift, is then given by:
\[ E[X_{t} - X_{t+1} \mid X_t] \geq \frac{\overline{e}}{\mu n^2e}\geq\frac{X_t}{\mu^2 n^2e}, \]
where the factor \( \frac{1}{\mu n^2} \) accounts for the probability of selecting the right vertex and making a beneficial rematch.

Given that $\binom{\mu}{2}2|R|$ is the maximum diversity, when all edges are pairwise distinct, it holds that \( X_0 \leq \binom{\mu}{2}2|R|\leq \mu^2|R| \leq |R|^3\leq m^{1.5}\leq n^3 \), the application of the multiplicative drift theorem yields the expected runtime of \( O(\mu^2 n^2\log(n)) \) to achieve maximum diversity.
\end{proof}
\Cref{thm:complete_bip_graph_small_gap} covers the case $\mu \geq |L|-|R|$ missing in the previous theorem, which gives a much larger runtime bound. Intuitively this happens because as $\mu$ gets greater than the gap between $|L|-|R|$ it is not longer guaranteed that we can always find a new rematch, such that this matching edge is not used by any other solution, thus making more than two bit flips necessary. \Cref{thm:sharp_comp_bip} includes such a situation with a theoretical lower bound.

\begin{theorem}\label{thm:sharp_comp_bip}
Let \( G=((L,R),E) \) be a complete bipartite graph with \( |R| < |L| \). Consider a population size \( \mu \), satisfying \( \mu < \frac{|R|}{2} \) and \( \mu \geq |L|-|R| \). There exists a starting population \( P_w \) such that when the \( (\mu + 1) \)-EA\(_D\) is applied to \( G \), the expected time to reach a population with maximal diversity is \( \Omega(m^{3.5}) \).
\end{theorem}

\begin{proof}
    Consider a bipartite graph \( G = (L \cup R, E) \) with vertex partitions \( L = \{l_1, l_2, \ldots, l_{|L|}\} \) and \( R = \{r_1, r_2, \ldots, r_{|R|}\} \). Define a matrix \( M \in \mathbb{R}^{\mu \times |R|} \) representing solutions to a matching problem, where each row of \( M \) corresponds to a solution, and each column \( j \) (for \( 1 \leq j \leq |R| \)) indicates the match in \( L \) for vertex \( r_j \) in \( R \).

The matrix \( M \) is constructed as follows:
\begin{enumerate}
    \item The first column of \( M \), denoted \( M_{*,1} \), is defined as:
    \[
    M_{*,1} = (l_{|L|}, l_{|L|}, l_{|L|-1}, \ldots, l_{|L|-\mu+2})^T.
    \]
    
    \item For each row \( i \) (for \( 1 \leq i \leq \mu \)), the entries in the row are filled by rotating the elements of \( L \) such that:
    \[
    M_{i,j} = l_{((j+i-2) \mod |L|)} \quad \text{for} \quad 2 \leq j \leq |R|.
    \]

    \item This process results in each row of \( M \) sharing the same sequence of vertices from \( L \), except for the first entry, with a cyclical shift to the right in each subsequent row.
\end{enumerate}

This matrix \( M \) represents distinct solutions for the bipartite graph matching problem, where each row corresponds to a different solution, and each column represents a match between a vertex in \( R \) and a vertex in \( L \), arranged according to the specified rotating pattern.

This matrix exemplifies the construction of solutions, with $\mu=5,|R|=11,|L|=12$ each row depicting a unique solution in the bipartite graph matching problem.

\[
M = \begin{pmatrix}
l_{12} & l_1 & l_2 & l_3 & l_4 & l_5 & l_6 & l_7 & l_8 & l_9\\
l_{12} & l_2 & l_3 & l_4 & l_5 & l_6 & l_7 & l_8 & l_9 &l_1 \\
l_{11} & l_3 & l_4 & l_5 & l_6 & l_7 & l_8 & l_9 & l_1 & l_2 \\
l_{10} & l_4 & l_5 & l_6 & l_7 & l_8 & l_9 & l_1 & l_2 & l_3 \\
l_9 & l_5 & l_6 & l_7 & l_8 & l_9 & l_1 & l_2 & l_3 & l_4 \\
\end{pmatrix}
\]

For each such matrix only the first column has two solutions using the same edge and the distance to optimal diversity is 2. Selecting any solution except these two can't increase the diversity. And for each of these 2 rows there is no value we can change the assignment of $r_1$ to without creating another duplicate edge or creating an invalid matching. Thus we have to change to one of the $|L|-\mu$ edges not part of the row and subsequently deactivate that edge and activate to one of the $|L|-|R|\leq\mu$ edges not used in the row. The probability of doing this is at most $\frac{2}{\mu}\frac{1}{m}\frac{|L|-\mu}{m}\frac{1}{m}\frac{|L|-|R|}{m}\leq \frac{2|R|\mu}{\mu m^4}\leq \frac{2}{m^{3.5}}$. The remaining runtime is $\Omega(m^{3.5})$.
\end{proof}
For the given hard instance, while there is no improving 2-bit flip there is however an improving 4-bit flip of the following form, changing two matches.  We make use of the fact that there is a match we can alter $(l_5\xrightarrow{}l_{10})$ freeing a vertex ($l_5$) we can match to $r_1$, which is unique in all solutions. \[
M = \begin{pmatrix}
\textcolor{red}{l_{5}} & l_1 & l_2 & l_3 & l_4 & \textcolor{red}{l_{10}} & l_6 & l_7 & l_8 & l_9\\
l_{12} & l_2 & l_3 & l_4 & l_5 & l_6 & l_7 & l_8 & l_9 &l_1 \\
l_{11} & l_3 & l_4 & l_5 & l_6 & l_7 & l_8 & l_9 & l_1 & l_2 \\
l_{10} & l_4 & l_5 & l_6 & l_7 & l_8 & l_9 & l_1 & l_2 & l_3 \\
l_9 & l_5 & l_6 & l_7 & l_8 & l_9 & l_1 & l_2 & l_3 & l_4 \\
\end{pmatrix}
\]
In the following theorem we generalize that such a 4-bit flip can always be found.
\begin{theorem}\label{thm:complete_bip_graph_small_gap}
For a complete bipartite graph \( G=((L,R),E) \) where \( |R| < |L| \), let the population size \( \mu \) satisfy \( \mu < \frac{|R|}{2} \) and \( \mu \geq |L|-|R| \). When the \( (\mu + 1) \)-EA\(_D\) is applied to \( G \), the expected time to achieve maximal diversity is bounded by \( O(\mu^2 m^4\log(m)) \).
\end{theorem}

\begin{proof}
We investigate the expected time for the $(\mu + 1)$-EA$_D$ to maximize diversity in a complete bipartite graph with the given conditions. Initially, we note that for any maximum matching there exist $|L|-|R|$ unmatched vertices from the left partition.

Let $M$ be a maximum matching in $G$. Consider that full diversity is not achieved yet and thus an edge $e_{rl} \in M$ is part of multiple maximum matchings. We define $\overline{L}\subseteq L$ to be the set of vertices in $L$ that are matched to a vertex $r \in R$ in at least one maximum matching. Given that $\mu < \frac{|R|}{2}<\frac{|L|}{2}$ and since a matching pairs each vertex in $R$ with at most one vertex in $L$, there must exist more than $\frac{|L|}{2}$ vertices in $L$ that are not paired with $u$ in any maximum matching. Let $R' \subseteq R$ be the set of vertices in $R$ that are adjacent to these unpaired vertices in $L$.

In the context of the $(\mu + 1)$-EA$_D$, by strategically reassigning the pairs in $M$, we can ensure an increase in diversity without decreasing the matching size. We denote by $M(r)$ the vertex in $L$ to which a vertex $r \in R$ is matched under $M$.

Now, for the sake of contradiction, assume that $\forall r' \in R': M(r') \in \overline{L}$. This would suggest that each vertex in $R'$ is matched to a vertex in $\overline{L}$ under $M$. However, since $\overline{L} < \mu < \frac{|R|}{2}$ and $|R'| > \frac{|R|}{2}$, this situation is not possible.

Therefore, there must exist a vertex $r' \in R'$ such that $M(r') \notin \overline{L}$. This implies that we can activate an edge connecting $r'$ with an unmatched vertex in $L$ and deactivate the edge currently matching $r'$ without reducing the size of the matching, thereby increasing diversity. Just as in \Cref{thm:complete_bip_graph_big_gap} each of those 4-bit flips only decreases or does not change the multiplicities of other edges, since they are both unique edges over all solutions. Also succsesively applying these 4-bit flips at most $\overline{e}$ times will result in optimal diversity,  so $\overline{e}\mu\geq X_t$ holds.

Define $X_t$ to be the difference between the optimal diversity and the current diversity at time $t$. Then, we observe a positive drift in the expected diversity increase per time step, similarly as \Cref{thm:complete_bip_graph_big_gap} which can be bounded below by:

\begin{align*}
E[X_{t} - X_{t+1} \mid X_t] &\geq \frac{\overline{e}}{\mu m^4}\left(1-\frac{1}{m}\right)^{m-4} \geq \frac{X_t}{\mu^2 m^4e}.
\end{align*}

Here, $\frac{1}{\mu}$ represents the probability of selecting the correct individual for reassignment, and the term $\frac{1}{m^4}\left(1-\frac{1}{m}\right)^{m-4}$ accounts for the probability of selecting the appropriate edges for activation and deactivation.

Given that $\binom{\mu}{2}2|R|$ is the maximum diversity, when all edges are pairwise distinct, it holds that \( X_0 \leq \binom{\mu}{2}2|R|\leq \mu^2|R| \leq |R|^3\leq m^{1.5} \), the Multiplicative Drift Theorem provides us with a runtime bound of $O(\mu^2 m^4\log(m))$ to achieve maximum diversity.
\end{proof}

A similar speedup as for the small gap case can be shown by applying the 2P-EA$_D$. 
\begin{theorem}\label{thm:complete_bip_graph_small_gap_efficient}
Given a complete bipartite graph \( G=((L,R),E) \) with \( |R| < |L| \), consider a population size \( \mu \) that fulfills \( \mu < \frac{|R|}{2} \) and \( \mu \geq |L|-|R| \). For the 2P-EA$_D$, the expected time to reach maximal diversity is \( O(\mu^2 m^2\log(m)) \).
\end{theorem}

\begin{proof}
Consider the $(\mu + 1)$-EA$_D$ applied to a complete bipartite graph $G=((L,R),E)$ under the condition $\mu \geq |L|-|R|$. Define the potential function \( X_t \) as in the previous theorem:
\[ X_t\coloneq \text{div}_{\text{opt}} - \text{div}(t). \]

In this adapted algorithm, we focus on efficiently increasing diversity by unmatching and then rematching only two vertices at a time. This process targets the subset of vertices in $R$ that can be rematched to different vertices in $L$ to increase diversity more effectively.

Let $\overline{e}$ be the number of edges that are shared across different matchings. The expected drift in \( X_t \) per step, considering the efficient selection and rematching process of only two vertices, is given by:
\[ E[X_{t} - X_{t+1} \mid X_t] \geq \frac{\overline{e}}{\mu n^2n^2}\left(1-\frac{1}{n}\right)^{n-2} \geq \frac{X_t}{\mu^2 n^4e}, \]
where the factor \( \frac{1}{\mu n^2} \) accounts for the probability of selecting the right solution and pair of vertices and $\frac{1}{n^2}$ of making a beneficial rematch. The term $\left(1-\frac{1}{n}\right)^{n-2}$ considers the probability of unmatching and rematching exactly two vertices without affecting the others.

Given that $\binom{\mu}{2}2|R|$ is the maximum diversity, when all edges are pairwise distinct, it holds that \( X_0 \leq \binom{\mu}{2}2|R|\leq \mu^2|R| \leq |R|^3\leq m^{1.5}\leq n^3 \), applying the Multiplicative Drift Theorem yields an expected runtime of \( O(\mu^2 n^4\log(n)) \) to achieve maximum diversity. Now since $|L|-|R|\leq\mu<\frac{|R|}{2}$ it holds that $|R|<|L|<1.5|R|$, which implies $O(|L|)=O(|R|)$. Also by definition $n=|L|+|R|$, so $O(n^2)=O(|L||R|)=O(m)$ and we get a bound of \( O(\mu^2 m^2\log(m)) \).
\end{proof}

\section{Runtime Analysis for paths}
This section introduces key theoretical results on paths. We commence with an introduction of useful notation to simplify the following proofs. Subsequently, we present a series of theorems that delineate the expected runtime to achieve this optimal diversity. These theorems compare the performance of the \((\mu+1)\)-EA\(_D\) and 2P-EA\(_D\) algorithms, providing a quantitative basis for assessing their efficacy.
\label{sec4}

In a path with an even number of edges, such as when \( m = 6 \), there are multiple ways to form a maximum matching. Each maximum matching includes exactly three edges, ensuring that no two edges in the matching share a vertex. The notation \( E^iO^j \) is used to represent these matchings, where \( i \) and \( j \) denote the number of edges with even and odd indices in the matching, respectively. The detailed proof is given in the following Lemma.
\begin{lemma}\label{lemma:number_path_matchings}
 The number of different maximum matchings on a path with $m$ edges is $\frac{m}{2}+1$ for $m$ even and $1$ for $m$ odd and each is of size $\lceil\frac{m}{2}\rceil$. Also for even $m$ each maximum matching can be described as $E^iO^{\frac{m}{2}-i}$. For $m$ odd the unique solution has the form $E^{\lceil\frac{m}{2}\rceil}O^{0}$.
\end{lemma}
\begin{proof}
We approach the proof of this lemma by employing induction to verify the claim regarding the number and arrangement of maximum matchings in path graphs of varying edge counts.\\
\textbf{Base Case ($m=1$,$m=2$):}\\
Clearly for $m=1$ there is only one solution consisting of one edge with index $0$, so the unique solution is $E^1O^0$. For $m=2$ only one of both edges of the path can be part of the maximum matching so the maximum matchings are $E^1O^0$ or $E^0O^1$.\\
\textbf{Inductive Step:}\\
In a maximum matching of size $m+1$ either the last or the second to last edge of the path has to be included, else we could increase the size by including the last edge.\\
Case 1: $(m+1)$ even\\
 If the last edge is part of the matching, then the first $m-1$ edges must also form a maximum matching, since the choice of being in the matching is independent of the last two edges. By the induction hypothesis we can extend each maximum matching on the $m-1$ edges by $O$.
If we instead include the second to last edge of the past, then the last and third to last edge of the path can't be part of the matching, while the remaining $m-2$ edges are independent of the choice and must thus also form a maximum matching. The remaining path is then odd and thus has a unique maximum matching , so inductively the only maximum matching of this form is $E^0O^{\frac{m+1}{2}}$
\\
Case 2: $(m+1)$ odd\\
 If the last edge is part of the matching, then by the induction hypothesis the maximum matching for the $m-1$ remaining edges is unique and thus the maximum matching including the last edge of even index is $E^0O^{\lceil\frac{m+1}{2}\rceil}$.\\
 If we instead include the second to last edge of the past, then the last and third to last edge of the path can't be part of the matching, while the remaining $m-2$ edges are independent of the choice and must thus also form a maximum matching. The remaining path is then even and by the induction hypothesis each matching will have $\frac{m-2}{2}+1$ edges, which is not maximum since by instead including the last edge we obtain a matching of size $\frac{m}{2}+1$.
\end{proof}

With an even number of edges, such as $m=6$, there are the following maximum matching configurations, represented as (matching edges in red)
\paragraph{Matching \( E^3O^0 \):}
\begin{tikzpicture}[baseline={([yshift=.5ex]current bounding box.center)}]
    \draw (1,0) -- (7,0); 
    \foreach \x in {1,...,7} {
        \draw[fill=black] (\x,0) circle (2pt);
    }
    \foreach \x [evaluate=\x as \evalx using int(\x-1)] in {1,...,6} {
        \draw (\x,0) -- node[below] {\evalx} (\x+1,0);
    }
    \draw[red, thick, shorten <=2pt, shorten >=2pt] (1,0) -- (2,0); 
    \draw[red, thick, shorten <=2pt, shorten >=2pt] (3,0) -- (4,0);
    \draw[red, thick, shorten <=2pt, shorten >=2pt] (5,0) -- (6,0);
\end{tikzpicture}

\paragraph{Matching \( E^2O^1 \):}
\begin{tikzpicture}[baseline={([yshift=.5ex]current bounding box.center)}]
    \draw (1,0) -- (7,0); 
    \foreach \x in {1,...,7} {
        \draw[fill=black] (\x,0) circle (2pt);
    }
    \foreach \x [evaluate=\x as \evalx using int(\x-1)] in {1,...,6} {
        \draw (\x,0) -- node[below] {\evalx} (\x+1,0);
    }
    \draw[red, thick, shorten <=2pt, shorten >=2pt] (1,0) -- (2,0); 
    \draw[red, thick, shorten <=2pt, shorten >=2pt] (3,0) -- (4,0);
    \draw[red, thick, shorten <=2pt, shorten >=2pt] (6,0) -- (7,0);
\end{tikzpicture}

\paragraph{Matching \( E^1O^2 \):}
\begin{tikzpicture}[baseline={([yshift=.5ex]current bounding box.center)}]
    \draw (1,0) -- (7,0); 
    \foreach \x in {1,...,7} {
        \draw[fill=black] (\x,0) circle (2pt);
    }
    \foreach \x [evaluate=\x as \evalx using int(\x-1)] in {1,...,6} {
        \draw (\x,0) -- node[below] {\evalx} (\x+1,0);
    }
    \draw[red, thick, shorten <=2pt, shorten >=2pt] (1,0) -- (2,0); 
    \draw[red, thick, shorten <=2pt, shorten >=2pt] (4,0) -- (5,0);
    \draw[red, thick, shorten <=2pt, shorten >=2pt] (6,0) -- (7,0);
\end{tikzpicture}

\paragraph{Matching \( E^0O^3 \):}
\begin{tikzpicture}[baseline={([yshift=.5ex]current bounding box.center)}]
    \draw (1,0) -- (7,0); 
    \foreach \x in {1,...,7} {
        \draw[fill=black] (\x,0) circle (2pt);
    }
    \foreach \x [evaluate=\x as \evalx using int(\x-1)] in {1,...,6} {
        \draw (\x,0) -- node[below] {\evalx} (\x+1,0);
    }
    \draw[red, thick, shorten <=2pt, shorten >=2pt] (2,0) -- (3,0); 
    \draw[red, thick, shorten <=2pt, shorten >=2pt] (4,0) -- (5,0);
    \draw[red, thick, shorten <=2pt, shorten >=2pt] (6,0) -- (7,0);
\end{tikzpicture}\\
With an odd number of edges, such as \( m = 5 \), there is only one maximum matching configuration, represented as
\paragraph{Matching \( E^3O^0 \):}
\begin{tikzpicture}[baseline={([yshift=.5ex]current bounding box.center)}]

    \draw (1,0) -- (6,0); 
    \foreach \x in {1,...,6} {
        \draw[fill=black] (\x,0) circle (2pt);
    }
    \foreach \x [evaluate=\x as \evalx using int(\x-1)] in {1,...,5} {
        \draw (\x,0) -- node[below] {\evalx} (\x+1,0);
    }
    \draw[red, thick, shorten <=2pt, shorten >=2pt] (1,0) -- (2,0); 
    \draw[red, thick, shorten <=2pt, shorten >=2pt] (3,0) -- (4,0);
    \draw[red, thick, shorten <=2pt, shorten >=2pt] (5,0) -- (6,0);
\end{tikzpicture} 
\\\\
In each case, every vertex is incident to at most one matching edge, and the \( E^iO^j \) notation describes the composition of the matching in terms of even and odd-indexed edges.

The following Lemma characterizes the conditions for maximal diversity within a population using this notation.

\begin{lemma}[Diversity of a Population]
\label{lemma:diverse_optimal_path}
The population with optimum diversity for even $\mu$ contains for each $j$ from $0$ to $\lfloor\frac{\mu}{2}\rfloor-1$ the individuals $E^{j}O^{\frac{m}{2}-j}$ and $E^{\frac{m}{2}-j}O^{j}$. For odd $\mu$ and $\lfloor\frac{\mu}{2}\rfloor\leq k\leq\frac{m}{2}-\lfloor\frac{\mu}{2}\rfloor$ it further contains any one individual of the form $E^kO^{\frac{m}{2}-k}$.
\end{lemma}
\begin{proof} We approach the proof of this lemma by employing induction to verify the claim regarding the number and arrangement of maximum matchings in path graphs of varying population sizes.
\hfill\\
\textbf{Base Case ($\mu = 1$, $\mu=2$):}\\
For $\mu = 1$ any solution maximizes the diversity of $0$.
For $\mu = 2$, the population with maximum diversity contains $E^0O^{\frac{m}{2}}$ and $E^{\frac{m}{2}}O^0$ with maximum diversity of $m$. Suppose that there exists another maximum matching population of size $2$, since the diversity has to be $m$, if the first matching is $M$ then the second matching must be the complement $M\setminus E$. As soon as edges with both even and odd indices are part of $M$, either $M$ or $M\setminus E$ does not have the form $E^iO^{\frac{m}{2}-i}$ and can't be a valid maximum matching by Lemma 12.\\
\textbf{Inductive Step:}\\
Suppose by way of contradiction that $E^0O^{\frac{m}{2}}$ is not part of the population. Then there exists an $i\geq 1$ such that for all solutions in the population the first $i$ edges with even index are part of the solution and the $(i+1)$th edge of one solution has odd index. By changing the $i$th edge of the solution to also be of odd index we would increase the diversity, which contradicts the assumption of having maximum diversity. Analogously this holds for   $E^{\frac{m}{2}}O^0$. Since all individuals are distinct all other solutions must start with an even edge and end with an odd edge. To the remaining $\mu-2$ individuals restricted on the inner $m-4$ edges we can then apply the Induction Hypothesis, so  for even $\mu$ the population further contains for each $j$ from $0$ to $\lfloor\frac{\mu}{2}\rfloor-2$ the individuals $$EE^{j}O^{\frac{m-4}{2}-j}O=E^{j+1}O^{\frac{m}{2}-j-1}$$ and $$EE^{\frac{m-4}{2}-j}O^{j}O=E^{\frac{m}{2}-j-1}O^{j+1}$$. For odd $\mu$ and $\lfloor\frac{\mu}{2}\rfloor\leq k\leq\frac{m}{2}-\lfloor\frac{\mu}{2}\rfloor$ it further contains any one individual of the form $E^{k+1}O^{\frac{m}{2}-k-1}$.

\end{proof}

Building up on this, in the following theorem we show that there is always a local improvement, needing $2$ bit flips, to improve diversity.
\begin{theorem} \label{theorem:max_matchings_paths}
In the $(\mu + 1)$-EA$_D$ applied to a path with $m$ edges, the expected time until the diversity is maximized is $O(\mu^3m^3)$.
\end{theorem}

\begin{proof}
We consider a path graph with an even number of edges $m$, where multiple maximum matchings are possible. The maximum matching is unique when $m$ is odd, hence the maximum diversity is trivially obtained in that case. Therefore, our analysis focuses on when $m$ is even.

Within a population, suppose there is duplication. By \Cref{lemma:diverse_optimal_path} it follows that there exists at least one individual for which the first $i \geq 0$ matched edges have even indices without another individual having the first $i+1$ matched edges with even indices, or an individual where the last $i \geq 0$ matched edges have odd indices without another individual having the last $i+1$ matched edges with odd indices.

Considering that the total number of distinct maximum matchings for a path with $m$ edges exceeds $\mu$, the likelihood of choosing an individual from the current population and correctly flipping two edges to enhance diversity is at least $\frac{1}{\mu}\frac{1}{m^2}(1-\frac{1}{m})^{m-2}$. This lower bound on the probability yields a diversity improvement of at least 1.

If the population has not reached maximal diversity but consists of pairwise distinct maximum matchings, then there must exist a maximal $0\leq j\leq \lfloor\frac{\mu}{2}-1\rfloor$ such that $E^jO^{\frac{m}{2}-j}$ or $E^{\frac{m}{2}-j}O^j$ is not present in the population. W.l.og. let this be $E^jO^{\frac{m}{2}-j}$. We focus on the individual $E^kO^{\frac{m}{2}-k},k<j$ with most odd edges. By applying a 2-bit flip we get $E^{k-1}O^{\frac{m}{2}-k+1}$. The diversity change, by replacing the parent, would be only determined by this edge change. This new odd edge is already used by $j$ matchings, since $j$ is maximal, and only those since else $E^kO^{\frac{m}{2}-k}$ would not have the most odd edges of the remaining population. By symmetry the deactivated even edge is used in $\mu-j$ solutions (excluding the parent). Thus the change in diversity by replacing the parent would be $\mu-j-j=\mu-2j$. By choice of $j$ this is strictly positive.  Since replacing the parent is possible, the diversity increase is at least of this size.
Let $X_t$ denote the difference between the optimal diversity and the current diversity at time $t$. The possibility of enhancing diversity via a two-bit flip provides us with a drift given by
$$E[X_{t}-X_{t+1} \mid X_t] \geq \frac{1}{\mu m^2}\left(1-\frac{1}{m}\right)^{m-2}\geq  \frac{1}{\mu m^2e}.$$
Since the initial diversity deficit $X_0$ is at most $m\mu^2$ (each pair of solutions can have a hamming distance of at most $m$), applying the additive drift theorem results in a runtime estimation of $O(\mu^3m^3)$.
\end{proof}
\begin{theorem} \label{theorem:max_matchings_paths_2p}
In the 2P-EA$_D$ applied to a path with $m$ edges, the expected time until the diversity is maximized is $O(\mu^3m^2)$.
\end{theorem}
\begin{proof}
Since the proof follows closely the arguments presented in \Cref{theorem:max_matchings_paths}, we will focus only on the different bounds on drift, which is the main differing element.

Any maximum matching $E^jO^{\frac{m}{2}-j},j>0$ can be chosen with probability $\frac{1}{\mu}$ and be mutated to $E^{j-1}O^{\frac{m}{2}-j+1}$  by unmatching the jth vertex and rematching him with probability $\frac{1}{2}$ to his unmatched left neighbour. Since all previous edges have to be of even index this neighbour must be unmatched. Analogously it  holds for $E^jO^{\frac{m}{2}-j},j<m-1$ to $E^{j+1}O^{\frac{m}{2}-j-1}$. For both the case of having duplicates or not being optimal in \Cref{theorem:max_matchings_paths} we make use of such a local edge swap. The drift is therefore given by
\[
E[X_{t}-X_{t+1} \mid X_t] \geq \frac{1}{\mu n2}\left(1-\frac{1}{n}\right)^{n-1}\geq \frac{1}{\mu n2e}.
\]
Where $\left(1-\frac{1}{n}\right)^{n-1}$ is the probability of not rematching any other vertex. Given that the initial diversity deficit $X_0$ is at most $m\mu^2$ (each pair of solutions can have a hamming distance of at most $m$), the additive drift theorem provides an upper bound on the expected run time of $O(\mu^3m^2)$, since $m=n-1$.
\end{proof}

\begin{figure*}
\begin{subfigure}{0.5\textwidth}
    \centering
    \includegraphics[width=\linewidth]{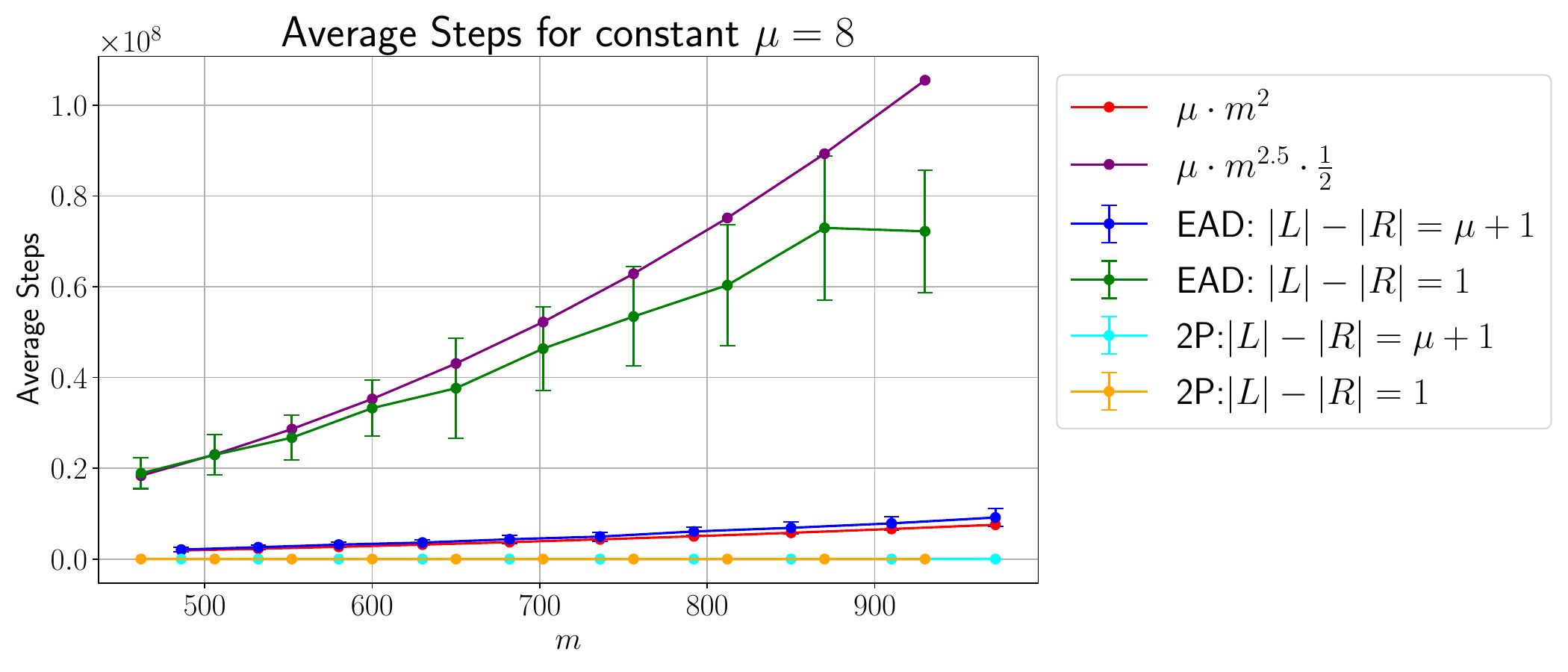}
    \caption{EA$_D$ and 2P with $\mu=8$ in Comp. Bip. Graphs}
    \label{fig:mu-fixed}
\end{subfigure}
\begin{subfigure}{0.5\textwidth}
    \centering
    \includegraphics[width=\linewidth]{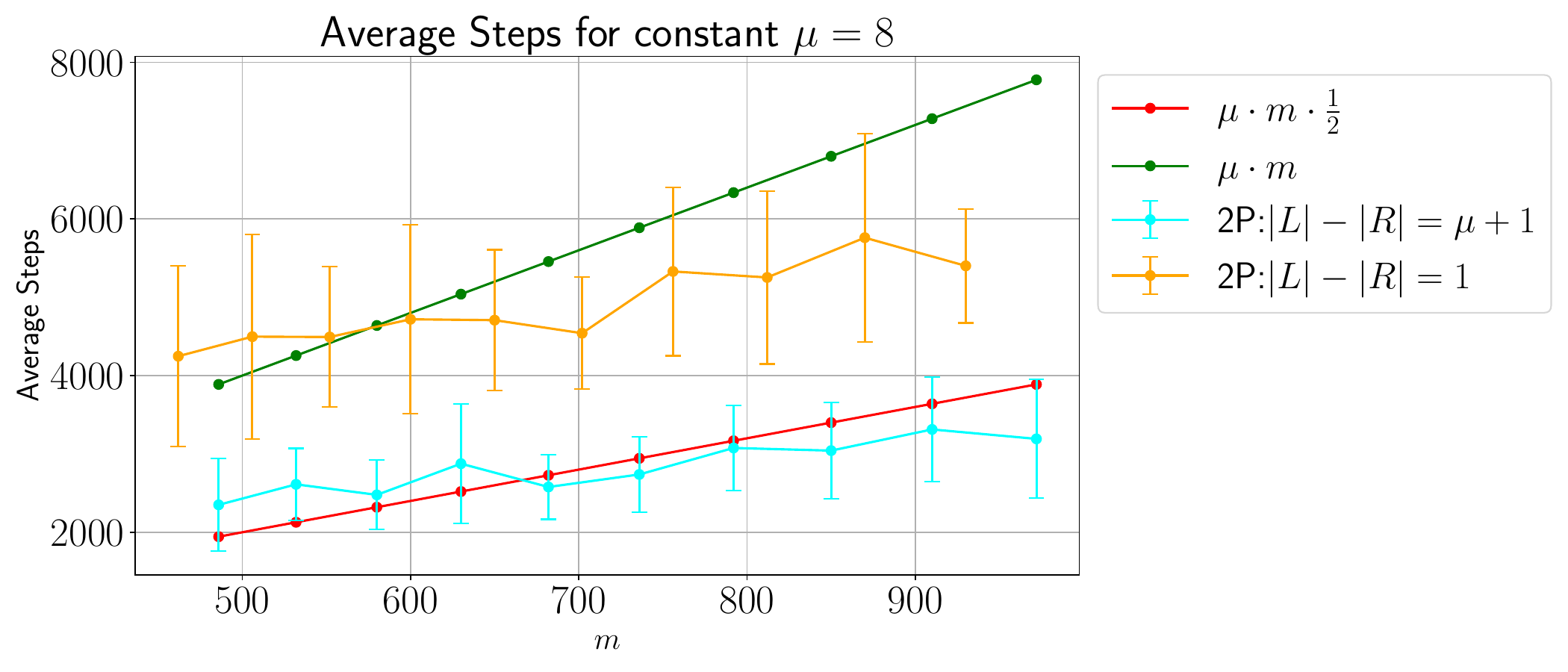}
    \caption{2P with fixed $\mu=8$ in Comp. Bip. Graphs}
    \label{fig:mu-fixed-2p}
\end{subfigure}
\\ 
\begin{subfigure}{0.5\textwidth}
    \centering
    \includegraphics[width=\linewidth]{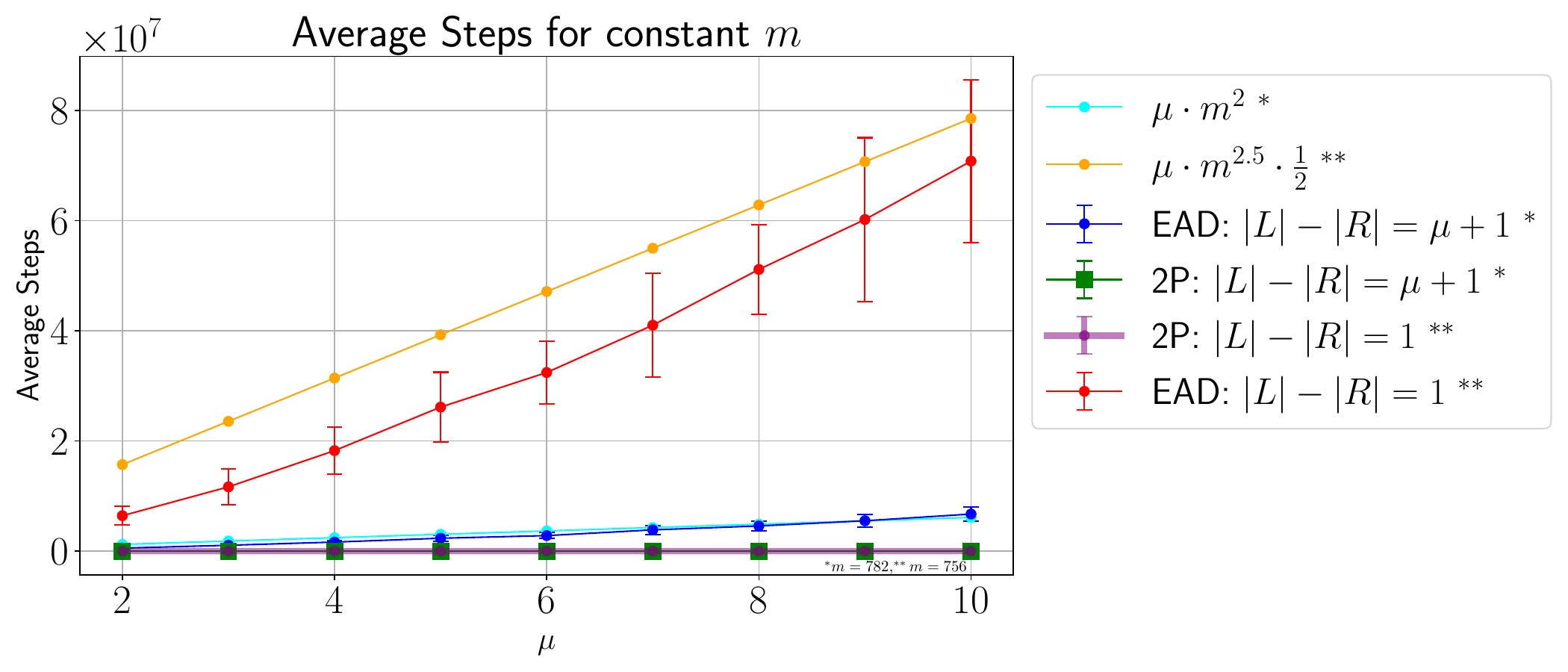}
    \caption{AED and 2P with fixed $m$ in Comp. Bip. Graphs}
    \label{fig:m-fixed}
\end{subfigure}
\hfill
\begin{subfigure}{0.5\textwidth}
    \centering
    \includegraphics[width=\linewidth]{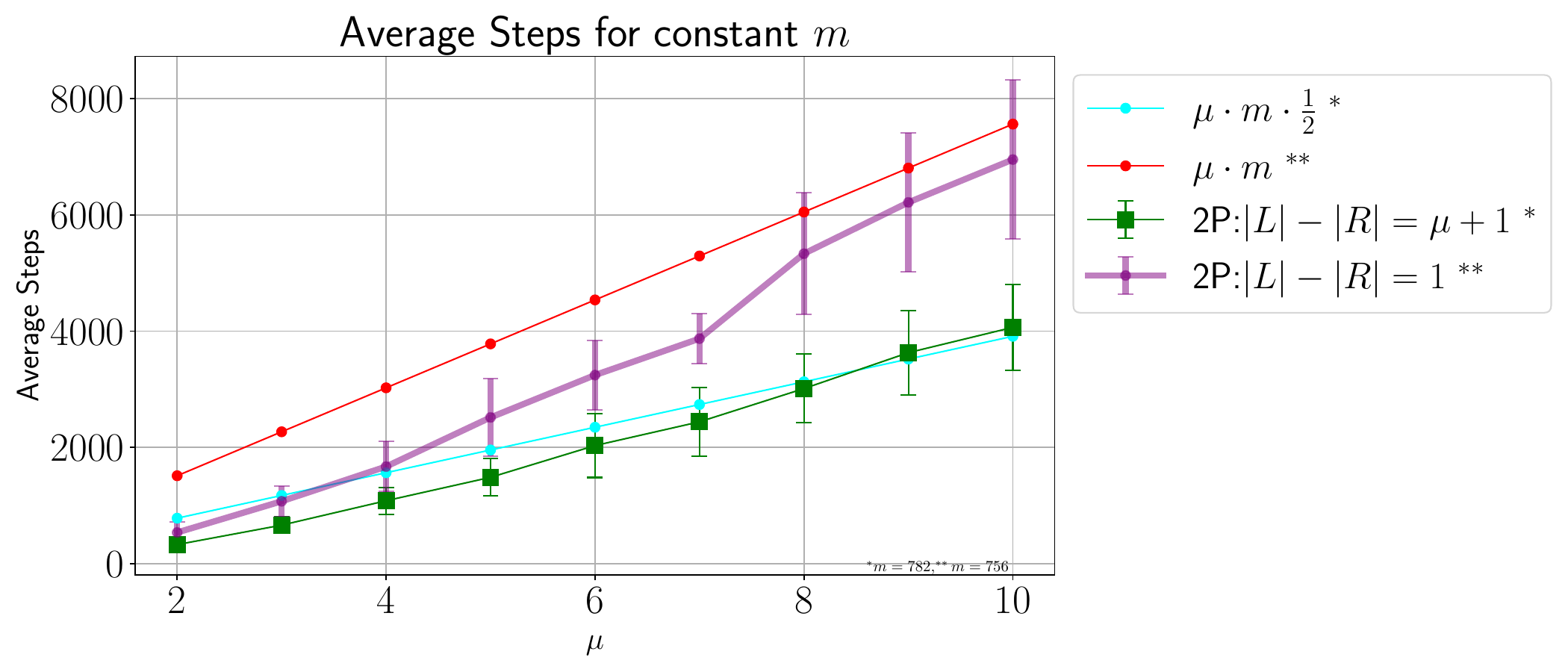}
    \caption{2P with fixed $m$ in Comp. Bip. Graphs}
    \label{fig:m-fixed-2p}
\end{subfigure}
\caption{Experimental results on complete bipartite graphs}
\end{figure*}

\section{Empirical Analysis}
\label{sec5}
In this section, we present our empirical findings on the performance of the evolutionary diversity algorithms on complete bipartite graphs and paths. Our experiments were designed to test the theoretical predictions made in previous sections, particularly focusing on the efficiency of the algorithm in terms of the number of iterations required to achieve optimal diversity.
\subsection{Experimental Setup}
\label{subsec:experimental-setup}

Our experiments were designed to explore the performance dynamics of the algorithms under two specific conditions: when the population size \(\mu\) is held constant and when the number of edges \(m\) remains fixed.

\paragraph{Complete Bipartite Graphs}
The starting condition for complete bipartite graphs involves a maximum matching where for each $0\leq i\leq |R|-1,$\(r_i\in R\) is matched to \(l_i\in L\), forming a homogeneous initial population. In the constant \(\mu\) scenario, we increase the size of both \(L\) and \(R\) by one unit per iteration to maintain a steady \(|L|-|R|\) difference, allowing a controlled analysis of the algorithms' scalability. In the constant $m$ scenario we simply increase $\mu$ by one per iteration.

\paragraph{Paths}
For paths, the initial population comprises maximum matchings including all even-indexed edges. With a fixed \(\mu\), the number of edges is incrementally increased by ten in each iteration, in order to cover a wider set of problem sizes, while staying experimentally feasible. In the constant $m$ case, out of feasibility, we simply increase $\mu$ by one per iteration.
\subsection{Methodology}
\label{subsec:methodology}

Each experiment was conducted 30 times to determine the average number of iterations and the standard deviation, estimating the algorithms' asymptotic runtime for both fixed population size (\(\mu\)) and a fixed number of edges (\(m\)). For complete bipartite graphs and fixed $m$ we chose $|L|=24$ and $|R|=23$ for the small gap case and $|L|=34$ and $|R|=23$ for the big gap case, such that the number of edges $m=782$ for the small gap case and $m=756$ for the big gap case are comparable in size.

\begin{figure*}
\centering
\begin{subfigure}{0.4\textwidth}
    \centering
    \includegraphics[width=\linewidth]{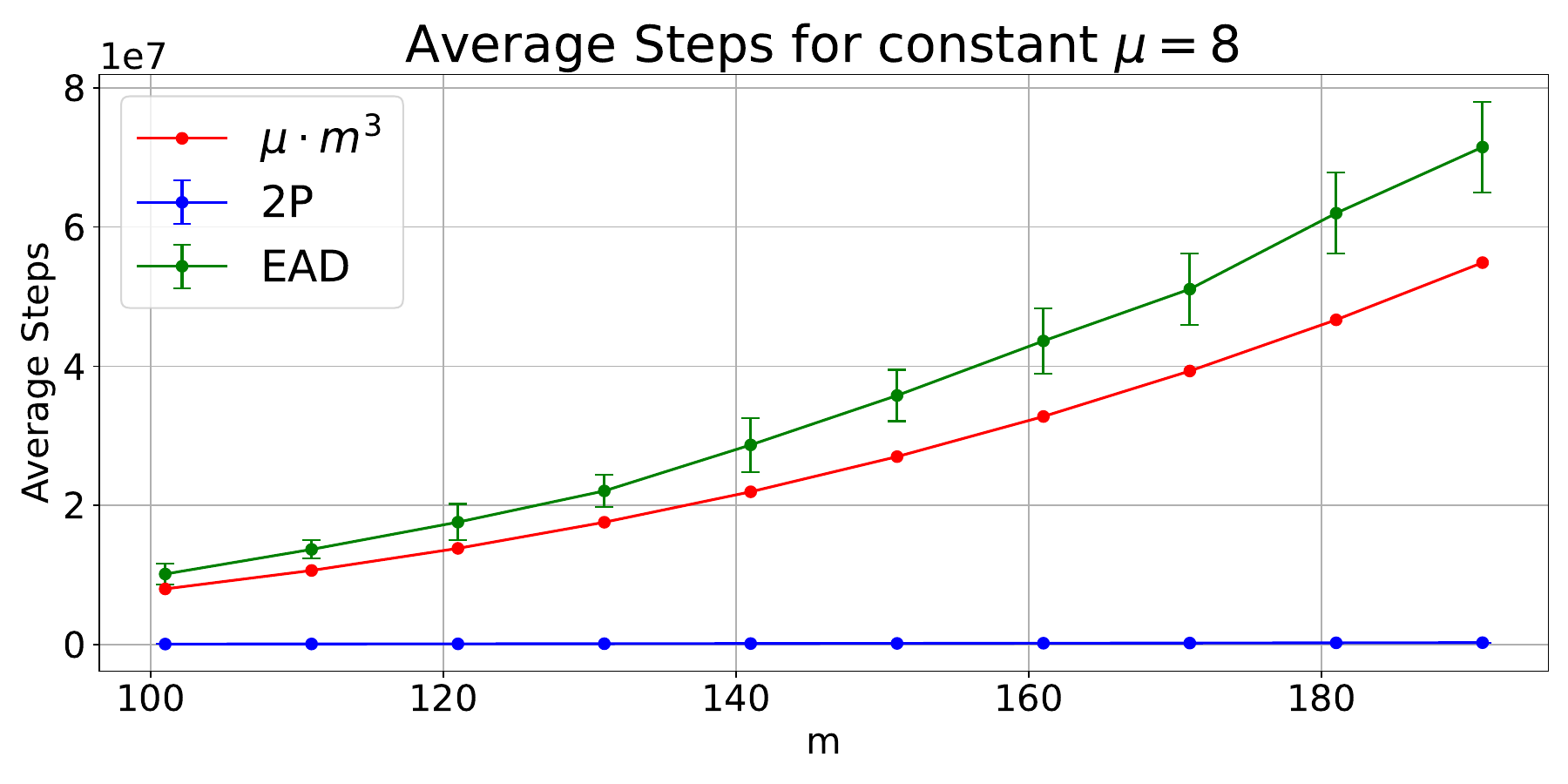}
    \caption{EA$_D$ and 2P with fixed $\mu=8$ in paths}
    \label{fig:mu-fixed-path}
\end{subfigure}
\begin{subfigure}{0.4\textwidth}
    \centering
    \includegraphics[width=\linewidth]{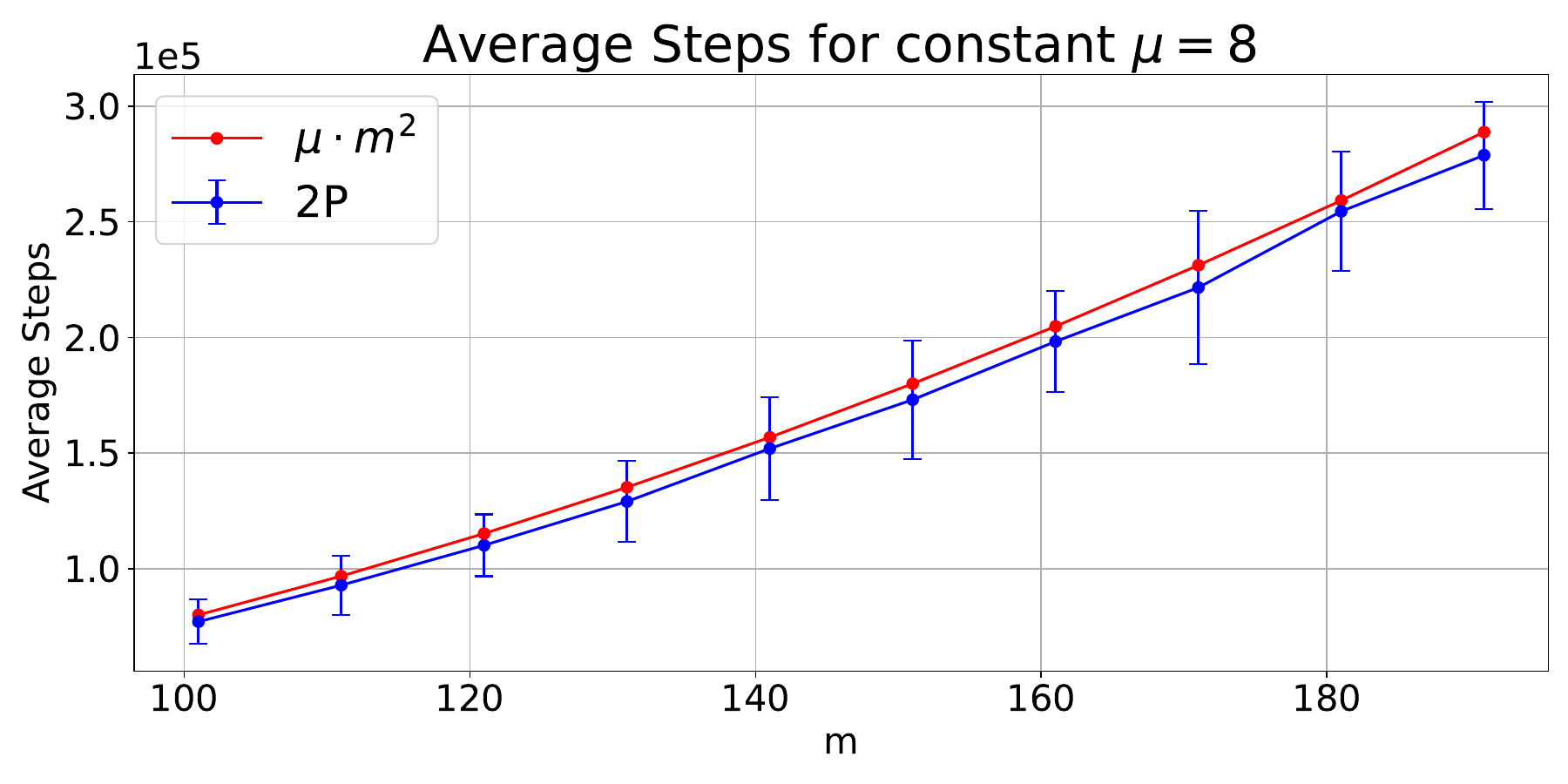}
    \caption{2P with fixed $\mu=8$ in paths}
    \label{fig:mu-fixed-2p-path}
\end{subfigure}
\\ 
\begin{subfigure}{0.4\textwidth}
    \centering
    \includegraphics[width=\linewidth]{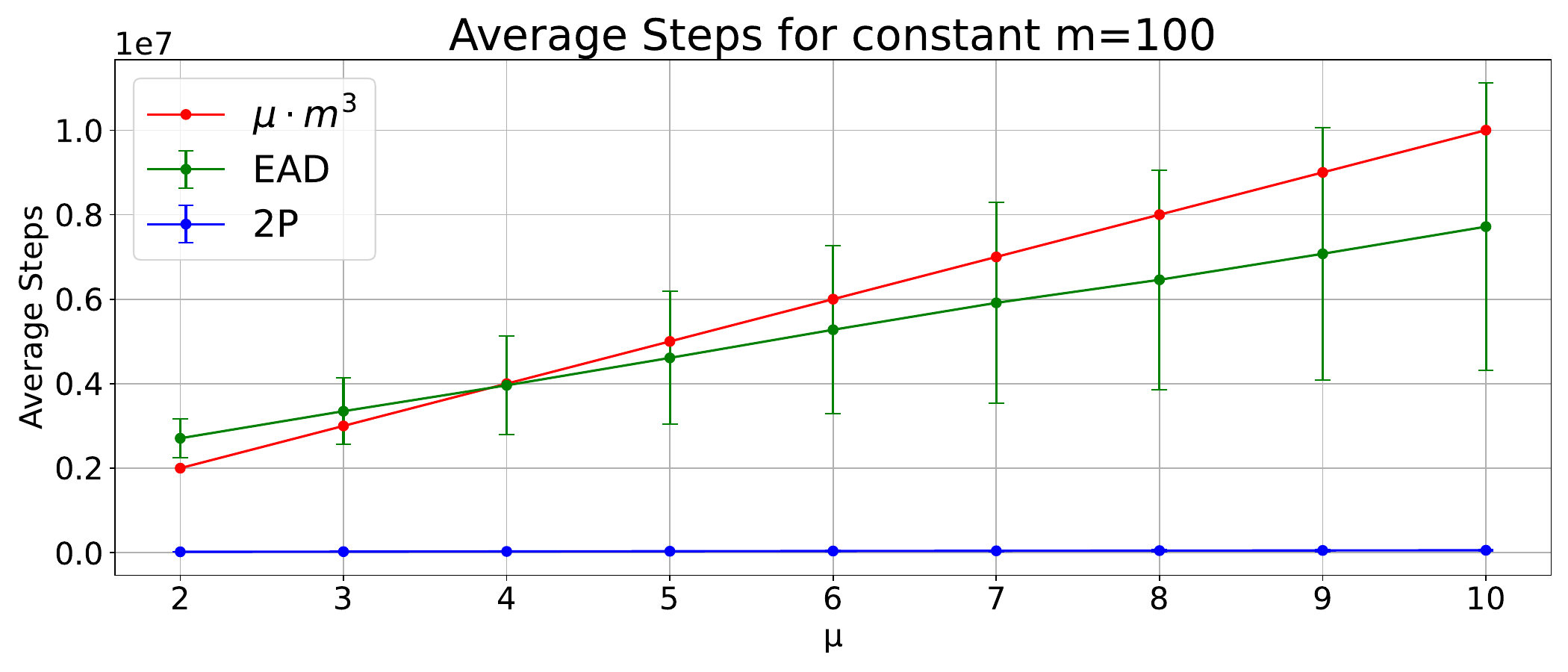}
    \caption{EA$_D$ and 2P with fixed $m=100$ in paths}
    \label{fig:m-fixed-path}
\end{subfigure}
\begin{subfigure}{0.4\textwidth}
    \centering
    \includegraphics[width=\linewidth]{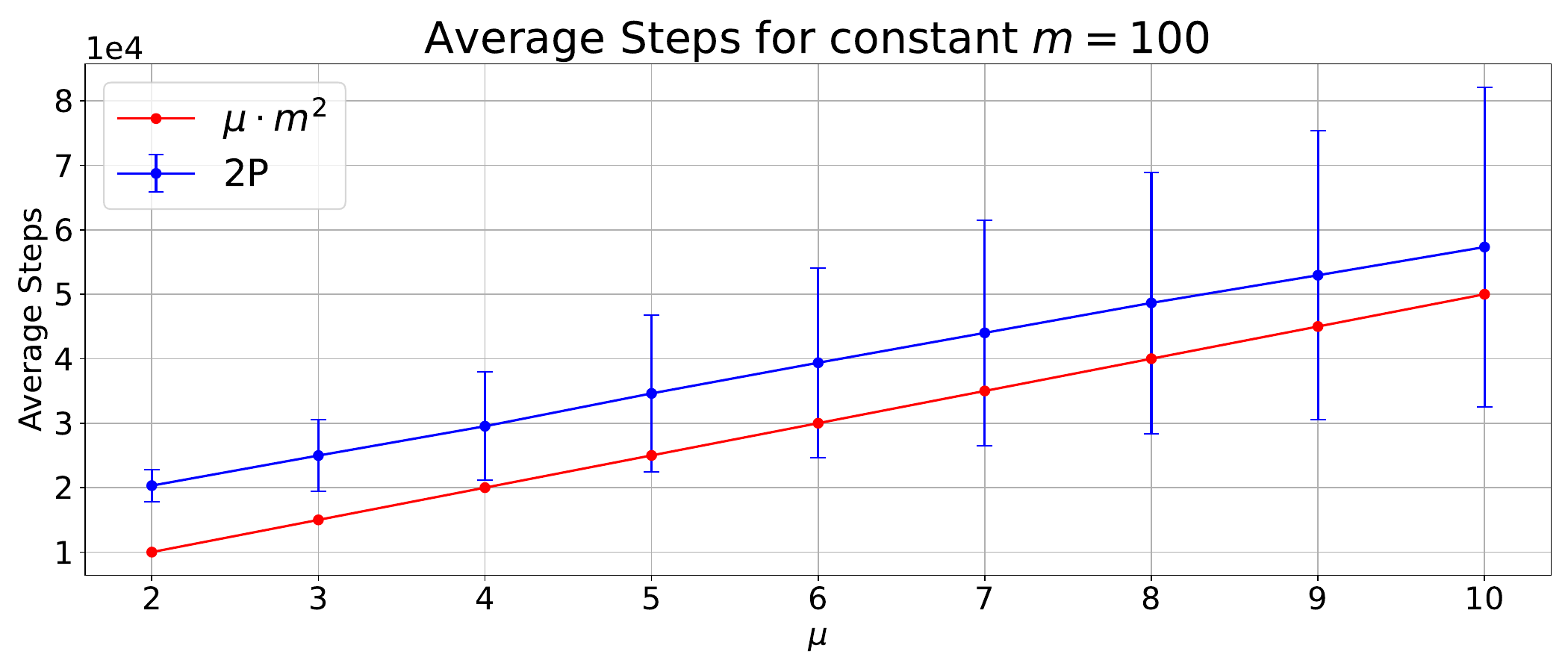}
    \caption{2P with fixed $m=100$ in paths}
    \label{fig:m-fixed-2p-path}
\end{subfigure}
\caption{Experimental results on paths}
\end{figure*}

\subsection{Complete Bipartite Graphs}
This subsection focuses on the performance of evolutionary diversity algorithms on complete bipartite graphs, specifically examining the \((\mu+1)\)-EA\(_D\) and 2P-EA\(_D\) algorithms.
\paragraph{\((\mu+1)\)-EA\(_D\)}
In Figure~\ref{fig:mu-fixed}, we show the average number of iterations for a fixed population size of $\mu=8$ and different values of $|L|-|R|$. Specifically, we examine cases where the difference \( |L|-|R| \) is either $1$,  referred to as the 'small gap' scenario or $\mu +1$, the 'big gap' scenario.
The \((\mu+1)\)-EA\(_D\) algorithm presented a quadratic growth in $m$ for the big gap case in iterations, empirically estimated as \(\mu m^2\), suggesting an out-performance by a factor of approximately \(\mu\log(m)\) over the theoretical bound. For the small gap case we empirically estimate the run time as \(\mu m^{2.5}\), an even stronger suggested out-performance by a factor of \(\mu m^{1.5}\log(m)\) when compared against the theoretical bound of \(O(\mu^2m^4\log(m))\).

In Figure~\ref{fig:m-fixed}, we display the average iteration counts for a constant edge count \( m \), considering the same values of \( |L|-|R| \). These findings echo the trends observed in Figure~\ref{fig:mu-fixed}, showcasing how the algorithm's behavior remains consistent across different graph sizes and population disparities.

\paragraph{2P-EA\(_D\)}

In Figure~\ref{fig:mu-fixed-2p} for $\mu$ fixed and Figure~\ref{fig:m-fixed-2p} for $m$ fixed, we zoom in on the results for the 2P-EA$_D$ algorithm. For both the small and big gap case the 2P-EA\(_D\) algorithm exhibited a linear increase in the number of iterations with respect to \(m\) when \(\mu\) was held constant and vice versa. Empirically, the run time for 2P-EA\(_D\) was observed to be close to \(\mu m\), a notable deviation from the predicted \(O(\mu^2m\log(m))\). 
The results summarized in Table~\ref{tab:runtime_results_bip} provide a summary of these observations. It is evident that the performance of the 2P-EA\(_D\) algorithm is not only superior in practice but also suggests that our theoretical bounds may be refined to more closely predict the empirical outcomes.
\subsection{Paths}
\label{subsec:paths}

This subsection focuses on the performance of evolutionary diversity algorithms on paths, specifically examining the \((\mu+1)\)-EA\(_D\) and 2P-EA\(_D\) algorithms.
\paragraph{\((\mu+1)\)-EA\(_D\)}
In Figure~\ref{fig:mu-fixed-path}, we present the average number of iterations when the population size \(\mu\) is fixed at 8. The graph illustrates how the number of iterations required for convergence changes as the number of edges \(m\) in the path increases. Figure~\ref{fig:m-fixed-path} shows the average number of iterations for a fixed number of edges \(m=100\) and varying population size \(\mu\). For the \((\mu+1)\)-EA\(_D\) algorithm, a trend of polynomial growth in the number of iterations is observed as a function of the problem size. When \(\mu\) is fixed at 8, the empirical runtime grows in line with \(\mu m^3\), which could indicate a performance better than the theoretical upper bound of \(O(\mu^3m^3)\) by a factor of $\mu^2$.

\paragraph{2P-EA\(_D\)}
When we examine the 2P-EA\(_D\) algorithm in Figure~\ref{fig:mu-fixed-2p-path} for a fixed \(\mu\), and in Figure~\ref{fig:m-fixed-2p-path} for a fixed \(m\), we notice a similar pattern. The empirical runtime for the 2P-EA\(_D\) is consistently around \(\mu m^2\), also possibly deviating by a factor of $\mu^2$ from the theoretical \(O(\mu^3m^2)\) bound.
The results in Table~\ref{tab:runtime_results_path} provide a summary of these observations. It is evident that the performance of the 2P-EA\(_D\) algorithm is not only superior in practice but also suggests that our theoretical bounds may be refined to more closely predict the empirical outcomes.

\begin{table}[t]
\centering 
\begin{small}
\caption{Summary of results for complete bipartite graphs}
\label{tab:runtime_results_bip}
\begin{tabular}{|c|c|c|c|c|}
\hline
Algo. & \multicolumn{2}{c|}{$|L|-|R|>\mu$} & \multicolumn{2}{c|}{$|L|-|R|\leq\mu$} \\ \hline
      & Empirical & Theor. UB & Empirical & Theor. UB \\ \hline
EA$_D$ & $\sim\mu m^2$ & $O(\mu^2m^2\log(m))$ & $\sim\mu m^{2.5}$ & $O(\mu^2m^4\log(m))$ \\ \hline
2P     & $\sim\mu m$   & $O(\mu^2n^2\log(n))$ & $\sim\mu m$       & $O(\mu^2m^2\log(m))$ \\ \hline
\end{tabular}
\end{small}
\end{table}

\begin{table}[t]
\centering
\caption{Summary of results for paths}
\label{tab:runtime_results_path}
\begin{small}
\begin{tabular}{|c|c|c|}
\hline
Algorithm & Empirical & Theor. UB\\ \hline
EA$_D$ & $\sim\mu m^3$ & $O(\mu^3m^3)$ \\ \hline
2P & $\sim\mu m^2$ & $O(\mu^3m^2)$ \\ \hline
\end{tabular}
\end{small}
\end{table}

\section{Conclusions}
\label{sec6}
In this study, we explored the application of evolutionary algorithms (EAs) for maximizing diversity in solving the maximum matching problem in complete bipartite graphs and paths. Our methodology was structured into two distinct phases: a rigorous theoretical analysis followed by comprehensive empirical evaluations. We specifically looked at the $(\mu+1)$-EA$_D$ and the Two-Phase Matching Evolutionary Algorithm (2P-EA$_D$), finding that both could achieve maximal diversity in expected polynomial time, with 2P-EA$_D$ showing a speed advantage in all scenarios.
Our findings not only underscore the utility of EAs in combinatorial diversity problems but also open up avenues for further research. A significant future direction would be to refine the theoretical upper bounds of these algorithms' runtime. Additionally, applying these insights to other graph problems and exploring real-world applications, could provide practical benefits.

\section*{Acknowledgements}
This work has been supported by the Australian Research Council through grant DP190103894. 

\bibliographystyle{splncs04}
\bibliography{sample-base}

\end{document}